\documentclass[10pt,twocolumn,journal]{IEEEtran}

\usepackage{xr}
\usepackage{cite}
\usepackage{amsmath,graphicx}
\usepackage[colorlinks=true,linkcolor=black,citecolor=black,urlcolor=black]{hyperref}
\usepackage{footnote}
\usepackage{booktabs}
\usepackage{tablefootnote}

\usepackage[margin=1in]{geometry}
\usepackage{amssymb}
\usepackage{mathrsfs}

\usepackage[mathscr]{eucal}
\usepackage{epsfig,epsf,psfrag}
\usepackage{amssymb,amsmath,amsfonts,latexsym}
\usepackage{amsmath,graphicx,bm,xcolor,url}
\usepackage[caption=false]{subfig} 
\usepackage{fixltx2e}%ordering of single and double column floats
\usepackage{array}%array and tabular environments
\usepackage{verbatim}
\usepackage{bm}
\usepackage{algpseudocode, cite}
\usepackage{algorithm}
\usepackage{verbatim}
\usepackage{textcomp}
\usepackage{mathrsfs}
\usepackage{epstopdf}
\usepackage{relsize}
\usepackage{cleveref} 
\usepackage{subfig}

\catcode`~=11 \def\UrlSpecials{\do\~{\kern -.15em\lower .7ex\hbox{~}\kern .04em}} \catcode`~=13 

\allowdisplaybreaks[3]

\newcommand{\calA}{\mathcal{A}}

\newcommand{\calM}{\mathcal{M}}
\newcommand{\calN}{\mathcal{N}}

\newcommand{\calS}{\mathcal{S}}
\newcommand{\calT}{\mathcal{T}}

\newcommand{\scrK}{\mathscr{K}}

\newcommand{\ba}{\mathbf{a}}
\newcommand{\bA}{\mathbf{A}}
\newcommand{\bb}{\mathbf{b}}
\newcommand{\bB}{\mathbf{B}}

\newcommand{\bC}{\mathbf{C}}

\newcommand{\bD}{\mathbf{D}}
\newcommand{\be}{\mathbf{e}}
\newcommand{\bE}{\mathbf{E}}

\newcommand{\bh}{\mathbf{h}}
\newcommand{\bH}{\mathbf{H}}

\newcommand{\bI}{\mathbf{I}}

\newcommand{\bL}{\mathbf{L}}

\newcommand{\bM}{\mathbf{M}}

\newcommand{\bp}{\mathbf{p}}

\newcommand{\bq}{\mathbf{q}}

\newcommand{\br}{\mathbf{r}}
\newcommand{\bR}{\mathbf{R}}
\newcommand{\bs}{\mathbf{s}}
\newcommand{\bS}{\mathbf{S}}
\newcommand{\bt}{\mathbf{t}}
\newcommand{\bT}{\mathbf{T}}

\newcommand{\bv}{\mathbf{v}}
\newcommand{\bV}{\mathbf{V}}
\newcommand{\bw}{\mathbf{w}}
\newcommand{\bW}{\mathbf{W}}
\newcommand{\bx}{\mathbf{x}}
\newcommand{\bX}{\mathbf{X}}

\newcommand{\bY}{\mathbf{Y}}
\newcommand{\bz}{\mathbf{z}}
\newcommand{\bZ}{\mathbf{Z}}

\newcommand{\rmF}{\mathrm{F}}

\newcommand{\bbE}{\mathbb{E}}

\newcommand{\bbN}{\mathbb{N}}

\newcommand{\bbP}{\mathbb{P}}

\newcommand{\bbR}{\mathbb{R}}

\DeclareMathAlphabet{\mathbsf}{OT1}{cmss}{bx}{n}
\DeclareMathAlphabet{\mathssf}{OT1}{cmss}{m}{sl}% slanted sans serif

\DeclareSymbolFont{bsfletters}{OT1}{cmss}{bx}{n}  
\DeclareSymbolFont{ssfletters}{OT1}{cmss}{m}{n}
\DeclareMathSymbol{\bsfGamma}{0}{bsfletters}{'000}
\DeclareMathSymbol{\ssfGamma}{0}{ssfletters}{'000}
\DeclareMathSymbol{\bsfDelta}{0}{bsfletters}{'001}
\DeclareMathSymbol{\ssfDelta}{0}{ssfletters}{'001}
\DeclareMathSymbol{\bsfTheta}{0}{bsfletters}{'002}
\DeclareMathSymbol{\ssfTheta}{0}{ssfletters}{'002}
\DeclareMathSymbol{\bsfLambda}{0}{bsfletters}{'003}
\DeclareMathSymbol{\ssfLambda}{0}{ssfletters}{'003}
\DeclareMathSymbol{\bsfXi}{0}{bsfletters}{'004}
\DeclareMathSymbol{\ssfXi}{0}{ssfletters}{'004}
\DeclareMathSymbol{\bsfPi}{0}{bsfletters}{'005}
\DeclareMathSymbol{\ssfPi}{0}{ssfletters}{'005}
\DeclareMathSymbol{\bsfSigma}{0}{bsfletters}{'006}
\DeclareMathSymbol{\ssfSigma}{0}{ssfletters}{'006}
\DeclareMathSymbol{\bsfUpsilon}{0}{bsfletters}{'007}
\DeclareMathSymbol{\ssfUpsilon}{0}{ssfletters}{'007}
\DeclareMathSymbol{\bsfPhi}{0}{bsfletters}{'010}
\DeclareMathSymbol{\ssfPhi}{0}{ssfletters}{'010}
\DeclareMathSymbol{\bsfPsi}{0}{bsfletters}{'011}
\DeclareMathSymbol{\ssfPsi}{0}{ssfletters}{'011}
\DeclareMathSymbol{\bsfOmega}{0}{bsfletters}{'012}
\DeclareMathSymbol{\ssfOmega}{0}{ssfletters}{'012}

\newcommand{\balpha}{\bm{\alpha}}

\newcommand{\bPi}{\bm{\Pi}}

\DeclareMathOperator{\rank}{rank}

\newtheorem{theorem}{Theorem} 
\newtheorem{lemma}[theorem]{Lemma}

\newenvironment{proof}[1][Proof]{\begin{trivlist}
\item[\hskip \labelsep {\bfseries #1}]}{\end{trivlist}}

\newcommand{\qednew}{\nobreak \ifvmode \relax \else
      \ifdim\lastskip<1.5em \hskip-\lastskip
      \hskip1.5em plus0em minus0.5em \fi \nobreak
      \vrule height0.75em width0.5em depth0.25em\fi}

\newcommand*{\QED}{\hfill\ensuremath{\blacksquare}}

\begin{document}
%\nocite{*}

\title{Model Selection for Nonnegative Matrix Factorization by Support Union Recovery}

\author{Zhaoqiang Liu
\thanks{The author is with with the Department of Electrical and Computer Engineering, NUS.}
}
%\address{Department of ECE, National University of Singapore}
%\thanks{M. Shell was with the Department
%of Electrical and Computer Engineering, Georgia Institute of Technology, Atlanta,

%\markboth{IEEE Transactions on Signal Processing,~Vol.~, No.~, 2016}%
%{Shell \MakeLowercase{\textit{Liu et al.}}: The Informativeness of $k$-Means for Learning Gaussian Mixture Models}

\maketitle

\begin{abstract}
Nonnegative matrix factorization (NMF) has been widely used in  machine learning and signal processing because of its non-subtractive, part-based property which enhances interpretability. It is often assumed that the latent dimensionality (or the number of components) is given. Despite the large amount of algorithms designed for NMF, there is little literature about automatic model selection for NMF with theoretical guarantees. In this paper, we propose an algorithm that first calculates an empirical second-order moment from the empirical fourth-order cumulant tensor, and then estimates the latent dimensionality by recovering the support union (the index set of non-zero rows) of a matrix related to the empirical second-order moment. By assuming a generative model of the data with additional mild conditions, our algorithm provably detects the true latent dimensionality. We show on synthetic examples that our proposed algorithm is able to find an approximately correct number of components. 
\end{abstract}
%
%\begin{keywords}
%Nonnegative matrix factorization, Model selection, Tensor method, Multiple measurement vector, Support union recovery
%\end{keywords}

\section{Introduction}
\label{sec:intro}

In a nonnegative matrix factorization (NMF) problem, we are given a data matrix $\bV \in \bbR^{F\times N}$, and we seek non-negative factor matrices $\bW \in \bbR^{F\times K}$, $\bH \in \bbR^{K \times N}$ such that a certain distance between $\bV$ and $\bW \bH$ is minimized. To reduce the data dimension and for the purpose of efficient computation, the integer $K$, which is said to be the latent dimensionality or the number of components, is usually chosen such that $K(F+N) \ll FN$. Since the publication of the seminar paper~\cite{lee2001algorithms} in 2000, NMF has been a popular topic in machine learning~\cite{cichocki2009nonnegative} and signal processing~\cite{buciu2008non}. There are many fundamental algorithms to approximately solve the NMF problem~\cite{lee2001algorithms, kim2008nonnegative, cichocki2007hierarchical} %\cite{lee2001algorithms, chu2004optimality, kim2008nonnegative, kim2008toward, kim2011fast, cichocki2007hierarchical, ho2011descent} 
with the implicit assumption that an effective number of the  latent dimensionality is known a priori. % good choice of the latent dimensionality $K$.

Despite the practical success of these fundamental algorithms, the estimation of the latent dimensionality remains an important issue. For example, researchers may wonder whether we can achieve better approximation accuracy with significantly less running time by selecting a better $K$ as the input of the algorithm. % However, in practice, it is often difficult to choose $K$ properly. 
Unfortunately, there is generally little literature discussing the model selection problem for NMF. Moreover, the methods proposed in papers about detecting latent dimensionality for NMF~\cite{tan2013automatic, cemgil2009bayesian, winther2007bayesian, liu65rank} either lack theoretical guarantees or require rather stringent conditions on the generative model of data.

\subsection{Main Contributions}\label{sec:main_contr}

We assume that each column $\bv$ of the data matrix $\bV = [\bv_1, \bv_2,\ldots,\bv_N] \in \bbR^{F\times N}$ is sampled from the following generative model
\begin{equation}\label{eq:nonnegICA}
 \bv = \bW \bh +\bz,
\end{equation}
where $\bW = [\bw_1,\bw_2,\ldots,\bw_K] \in \bbR^{F\times K}_+$ is the mixing matrix (or the ground-truth non-negative dictionary matrix) and we assume that $\rank(\bW) = K$. $\bh \in \bbR^{K}$ is a latent random vector with independent coordinates\footnote{Similar to that in~\cite{li2016recovery}, we will not require $\bh$ to be non-negative.}, and $\bz \in \bbR^{F}$ is a multivariate Gaussian random vector. $\bz$ is assumed to be independent with $\bh$. We write $\bH = [\bh_1,\bh_2,\ldots,\bh_N] \in \bbR^{K\times N}$. In the context of this generative model, our goal is to find the number of columns of $\bW$ from the observed matrix $\bV$. This generative model can be viewed as a non-negative variant of that for independent component analysis (ICA)~\cite{hyvarinen2004independent}. 

For the data matrix $\bV \in \bbR^{F\times N}$ generated from the above model, we first calculate an empirical second-order moment, denoted as $\hat{\bM}_2 \in \bbR^{F\times F}$, from the empirical fourth-order cumulant tensor. We prove that $\hat{\bM}_2$ approximates its expectation, denoted as $\bM_2$, well with high probability when $N$ is sufficiently large.  We also show that $\bM_2$ can be written as $\bM_2 = \bM_2 \bX^*$, where $\bX^* \in \bbR^{F \times F}$ contains exactly $K$ non-zero rows. Finally, we prove that under certain conditions, an $\ell_1/\ell_2$ block norm minimization problem (cf. \eqref{eq:NMF_MMV} to follow) over $\hat{\bM}_2$ is able to detect the correct number of column of $\bW$ from the recovery of a support union.

\subsection{Notations}\label{sec:notations}

We use capital boldface letters to denote matrices and we use lower-case boldface letters to denote vectors. We use $a_{ij}$ or $[\bA]_{ij}$ to denote the $(i,j)$-th entry of $\bA$. $\left[N\right]$ represents $\{1,2,\cdots,N\}$ for any positive integer $N$. For $\bX \in \bbR^{L \times M}$ and any $l \in [L]$, $m\in[M]$, we use $\underline{\bx}_l$, $\bx_m$ to denote the $l$-th row and the $m$-th column of $\bX$, respectively. We write $\underline{\bV}_\mathscr{K} := \bV(\mathscr{K},\colon)$ as the rows of $\bV$ indexed by $\mathscr{K}$, and $\bV_\mathscr{K} :=\bV \left(\colon,\mathscr{K}\right)$ denotes the columns of $\bV$ indexed by $\mathscr{K}$.
%$\bV\left(i,\colon\right)$ denotes the $i$-th row of $\bV$, $\bV\left(\colon,j\right)$ denotes the $j$-th column of $\bV$ and $\bV \left(\colon,\mathscr{K}\right)$ denotes the columns of $\bV$ indexed by $\mathscr{K}$. 
$\|\bV\|_1, \|\bV\|_2, \|\bV\|_{\infty}, \|\bV\|_{\rmF}$ represents the $1$-norm, the spectral norm, the infinity norm and the Frobenius norm of $\bV$, respectively. %Inequality $\bv\geq 0$ or $\bV \geq 0$ denotes element-wise nonnegativity. %Let $\bV_{1} \in \mathbb{R}^{F\times N_{1}}$ and $\bV_{2} \in \mathbb{R}^{F\times N_{2}}$, we denote by $\left[\bV_{1}, \bV_{2}\right]$ the horizontal concatenation of the two matrices. 
Let $\bV_{1} \in \mathbb{R}^{F_{1}\times N}$ and $\bV_{2} \in \mathbb{R}^{F_{2}\times N}$. We denote by $\left[\bV_{1}; \bV_{2}\right]$ the vertical concatenation of the two matrices. $\mathrm{Diag}(\bw)$ represents the diagonal matrix whose diagonal entries are given by $\bw$. The support of a vector $\bx$ is denoted as $\mathrm{supp}(\bx):=\{i:x_i \ne 0\}$. The support union of a matrix $\bX$ with $N$ columns is defined as $\mathrm{Supp}(\bX):=\cup_{n=1}^N \mathrm{supp}(\bx_n)$.

% \section{Problem Formulation}\label{sec:prob_form} 
\section{Tensor Methods}\label{sec:tensor_methods}

In this section, we calculate an empirical second moment $\hat{\bM}_2$ using a tensor method, and we prove that the empirical second moment is close to its expectation $\bM_2$ with high probability when the sample size $N$ is sufficiently large. 

\subsection{The Derivation of $\bM_2$ and $\hat{\bM}_2$}

Let $\bv$ be a random vector corresponding to the generative model~\eqref{eq:nonnegICA} with $\bbE[h_k] =0$ and $\bbE[z_f] =0$ for $k \in[K], f\in[F]$. We have the following lemma which says that $\bM_2$ can be written in a nice form.
\begin{lemma}[\cite{comon2010handbook, anandkumar2014tensor}]\label{lem:tensor_ica}
 Define 
 \begin{equation}
  \calM_4 := \bbE[\bv \otimes \bv \otimes \bv \otimes \bv] - \calT,
 \end{equation}
where $[\bv \otimes \bv \otimes \bv \otimes \bv]_{ijlm} = v_i v_j v_l v_m$ and $\calT$ is the fourth-order tensor with
\begin{align}
 [\calT]_{ijlm} &:= \bbE[v_{i}v_{j}]\bbE[v_{l}v_{m}] + \bbE[v_{i}v_{l}]\bbE[v_{j}v_{m}] \nonumber\\
 &+ \bbE[v_{i}v_{m}]\bbE[v_{j}v_{l}]
\end{align}
for all $i,j,l,m \in [F]$.
Let\footnote{It is implicitly assumed in~\cite{anandkumar2014tensor} that $\mathrm{Var}[h_k]=1$, and thus $\kappa_k = \bbE[h_k^4]-3$.} $\kappa_k = \bbE[h_k^4]-3\bbE[h_k^2]$ for each $k \in [K]$. Then 
\begin{equation}
 \calM_4 = \sum_{k=1}^K \kappa_k \bw_k \otimes \bw_k \otimes \bw_k \otimes \bw_k.
\end{equation}
%Note that $\kappa_k$ corresponds to the excess kurtosis, a measure of non-Gaussianity as $\kappa_k =0$ if $h_k$ is a standard normal random variable. $\bW$ is not identifiable if $\bh$ is a multivariate Gaussian. 
In addition, we have that 
\begin{equation}
 \bM_2 := \calM_4 (\bI,\bI,\bs,\bt) = \sum_{k=1}^K \kappa_k (\bs^T \bw_k) (\bt^T \bw_k) \bw_k \otimes \bw_k 
\end{equation}
for any $\bs,\bt \in \bbR^{F}$. Here for matrices $\bV_1 \in \bbR^{F \times F_1},\bV_2 \in \bbR^{F\times F_2},\bV_3 \in \bbR^{F \times F_3},\bV_4 \in \bbR^{F \times F_4}$, $\calM_4(\bV_1,\bV_2,\bV_3,\bV_4)$ is defined as the tensor whose $(i_1,i_2,i_3,i_4)$-th entry is
\begin{align}
 & \left[\calM_4(\bV_1,\bV_2,\bV_3,\bV_4)\right]_{i_1,i_2,i_3,i_4} = \nonumber\\ 
 & \sum_{j_1,j_2,j_3,j_4 } \left[\calM_4\right]_{j_1,j_2,j_3,j_4} [\bV_1]_{j_1,i_1} [\bV_2]_{j_2,i_2} [\bV_3]_{j_3,i_3} [\bV_4]_{j_4,i_4}.
\end{align}
\end{lemma}

We calculate $\hat{\bM}_2$ %which is an approximation matrix to $\bM_2$  
from the sample matrix $\bV$. Let 
\begin{equation}
 \hat{\calM}_4 := \frac{\sum_{n=1}^N \bv_n \otimes \bv_n \otimes \bv_n \otimes \bv_n}{N} - \hat{\calT},
\end{equation}
where for $1 \le i,j,l,m \le F$,
\begin{align}
 &[\hat{\calT}]_{ijlm} = \frac{\sum_{n=1}^N v_{i,n}v_{j,n}}{N} \frac{\sum_{n=1}^N v_{l,n}v_{m,n}}{N} + \nonumber\\
 &\frac{\sum_{n=1}^N v_{i,n}v_{l,n}}{N} \frac{\sum_{n=1}^N v_{j,n}v_{m,n}}{N} + \frac{\sum_{n=1}^N v_{i,n}v_{m,n}}{N} \frac{\sum_{n=1}^N v_{j,n}v_{l,n}}{N}.
\end{align}
Denoting $\hat{\bM}_2$ as 
\begin{equation}
 \hat{\bM}_2 = \hat{\calM}_4 (\bI,\bI,\bs,\bt).
\end{equation}
%That is, for $1 \le i,j \le F$, 
%\begin{equation}
% [\hat{\bM}_2]_{ij} = \sum_{l=1}^F \sum_{m=1}^F [\hat{\calM}_4]_{ijlm} s_l t_m.
%\end{equation}
We have that $\bbE[\hat{\bM}_2] = \bM_2$. For simplicity, we take $\bs=\bt = \be \in \bbR^{F}$, where $\be$ is the vector of all ones. For any $k \in [K]$, because $\bw_k \ne \mathbf{0}$, we have that $\be^T\bw_k >0$. In addition, if $\kappa_k \ne 0$, let $\alpha_k = \kappa_k (\be^T\bw_k)^2$, we have $\alpha_k \ne 0$ and 
\begin{equation}\label{eq:M2_form1}
 \bM_2 = \sum_{k=1}^K \alpha_k \bw_k \bw_k^T.
\end{equation}
Moreover, now we have that 
\begin{equation}
 [\hat{\bM}_2]_{ij} = \sum_{l=1}^F \sum_{m=1}^F [\hat{\calM}_4]_{ijlm}.
\end{equation}

\subsection{Bounding the Distance between $\bM_2$ and $\hat{\bM}_2$}

Let $\bE = \hat{\bM}_2 - \bM_2$, and assume that all the coordinates of $\bh$ are identically and independently distributed with $m_p := \bbE [h_k^p]$ for all $k \in [K]$, $p \in \bbN$. In particular, we assume that $m_1 = 0$ and $m_4 \ne 3m_2$. Let $M_8 = m_8 + m_7 m_1 + m_6 m_2 +m_6m_1^2 + \ldots + m_2^4 + m_2^3m_1^2 + m_2^2m_1^4+m_2 m_1^6 +m_1^8$, $M_4 = m_4 + m_3m_1 +m_2^2 + m_2 m_1^2 +m_1^4$ and $M = \max\{M_8,M_4^2\}$. Denote $W_{\max}$ as $W_{\max}:= \max_{f,k}{w_{fk}}$. Suppose that $\bz \sim \calN(\mathbf{0},\sigma^2\bI)$ and let $\Delta = \max\{\sigma,1\}$. From the following lemma, we can see that if $N$ is sufficiently large, the distance between $\bM_2$ and $\hat{\bM}_2$ (with respect to Frobenius norm) is sufficiently small with high probability. 
\begin{lemma}\label{lem:bd_bE}
 For any $\delta \in (0,1)$, we have that with probability at least $1-\delta$,
 \begin{equation}
  \|\bE\|_\rmF < \frac{117 \sqrt{70M} W_{\max}^4 K^4 \Delta^4 F^3 }{\sqrt{\delta N}}.
 \end{equation}
\end{lemma}
\begin{proof}
 We have 
 \begin{align}
  \bbE[\|\bE\|_\rmF^2] & = \sum_{i \ne j} \bbE[e_{ij}^2] + \sum_{i} \bbE[e_{ii}^2] \\
  & = \sum_{i \ne j} \mathrm{Var}[e_{ij}] + \sum_{i} \mathrm{Var}[e_{ii}] \\
  & = \sum_{i \ne j} \mathrm{Var}\left[\sum_{l,m} [\hat{\calM}_4]_{ijlm}\right] \nonumber\\
  & + \sum_{i} \mathrm{Var}\left[\sum_{l,m}[\hat{\calM}_4]_{iilm}\right].
 \end{align}
 First, we consider $\mathrm{Var}\left[[\hat{\calM}_4]_{ijlm}\right]$ for $i,j,l,m \in [F]$. Let $A = \frac{\sum_{n} v_{in}v_{jn}v_{ln}v_{mn}}{N}$, $B_1 = \frac{(\sum_{n} v_{in}v_{jn})(\sum_{n} v_{ln}v_{mn})}{N^2}$, $B_2 = \frac{(\sum_{n} v_{in}v_{ln})(\sum_{n} v_{jn}v_{mn})}{N^2}$ and $B_3 = \frac{(\sum_{n} v_{in}v_{mn})(\sum_{n} v_{jn}v_{ln})}{N^2}$. Let $U = 5670 W_{\max}^8 K^8 \Delta^8 M$, we have that for any $n \in [N]$, 
 \begin{align}
  \mathrm{Var}[A] & = \frac{1}{N}\mathrm{Var}\left[v_{in}v_{jn}v_{ln}v_{mn}\right] \le \frac{1}{N}\bbE\left[v_{in}^2 v_{jn}^2 v_{ln}^2 v_{mn}^2\right] \\
  %& =\frac{1}{N}\bbE\left[\left(\sum_{k} h_{kn}w_{ik} + z_{in}\right)^2 \left(\sum_{k} h_{kn}w_{jk} + z_{jn}\right)^2 \left(\sum_{k} h_{kn}w_{lk} + z_{ln}\right)^2 \left(\sum_{k} h_{kn}w_{mk} + z_{mn}\right)^2\right] \\
  & \le \frac{1}{N} 3^4 \binom{8}{4} W_{\max}^8 K^8 \Delta^8 M_8 \\
  & \le \frac{5670 W_{\max}^8 K^8 \Delta^8 M}{N} = \frac{U}{N}.
 \end{align}
In addition, we have 
\begin{align}
 \mathrm{Var}[B_1] & = \frac{1}{N^4}\mathrm{Var}\left[(\sum_{n_1} v_{i,n_1}v_{j,n_1})(\sum_{n_2} v_{l,n_2}v_{m,n_2})\right]. 
\end{align}
When $n_1 \ne n_2$, similarly, we have 
\begin{align}
 & \mathrm{Var}\left[v_{i,n_1}v_{j,n_1}v_{l,n_2}v_{m,n_2}\right] \le \bbE\left[v_{i,n_1}^2 v_{j,n_1}^2 v_{l,n_2}^2 v_{m,n_2}^2\right] \\
 & =\bbE\left[v_{i,n_1}^2 v_{j,n_1}^2\right] \bbE\left[v_{l,n_2}^2 v_{m,n_2}^2\right] \\
 & \le \left(3^2 \binom{4}{2} W_{\max}^4 K^4 \Delta^4 M_4\right)^2 \\
 & \le U.
\end{align}
In addition, by $|\mathrm{Cov}(X,Y)| \le \sqrt{\mathrm{Var}(X)\mathrm{Var}(Y)}$, the absolute value of each covariance in $\mathrm{Var}[B_1]$ is also upper bounded by $U$. Furthermore, there are less than $16N^3-N^2$ non-zero covariance terms in $\mathrm{Var}[B_1]$. Therefore, we have 
\begin{equation}
 \mathrm{Var}[B_1] \le \frac{1}{N^4} 16N^3 U = \frac{16U}{N}.
\end{equation}
Symmetrically, we have that $\mathrm{Var}[B_2]$ and $\mathrm{Var}[B_3]$ are also upper bounded by $\frac{16U}{N}$. Then 
\begin{align}
 & \mathrm{Var}\left[[\hat{\calM}_4]_{ijlm}\right] = \mathrm{Var}\left[A + B_1 + B_2 + B_3\right] \\
 & \le \left(\sqrt{\mathrm{Var}[A]}+\sqrt{\mathrm{Var}[B_1]}+\sqrt{\mathrm{Var}[B_2]}+\sqrt{\mathrm{Var}[B_3]}\right)^2 \\
 & \le \frac{169U}{N}.
\end{align}
For any $1\le i,j\le F$, we have 
\begin{equation}
 \bbE[e_{ij}^2] = \mathrm{Var}\left[\sum_{l,m} [\hat{\calM}_4]_{ijlm}\right] \le \frac{169UF^4}{N}.
\end{equation}
Thus 
\begin{equation}
 \bbE\left[\|\bE\|_\rmF^2\right] \le \frac{169UF^6}{N}.
\end{equation}
By Markov inequality, we have that for any $\delta \in (0,1)$,
\iffalse
\begin{equation}
 \bbP(e_{ij}^2 \ge \frac{1}{\delta} \bbE[e_{ij}^2]) \le \delta.
\end{equation}
That is, with probability at least $1-\delta$, 
\begin{equation}
 |e_{ij}| \le \frac{\sqrt{\bbE[e_{ij}^2]}}{\sqrt{\delta}} \le \frac{13 F^2 \sqrt{U}}{\sqrt{\delta N}}.
\end{equation}
\fi
\begin{equation}
 \bbP(\|\bE\|_\rmF^2 \ge \frac{1}{\delta} \bbE\left[\|\bE\|_\rmF^2\right]) \le \delta.
\end{equation}
That is, with probability at least $1-\delta$, 
\begin{align}
 \|\bE\|_\rmF &\le \frac{\sqrt{\bbE\left[\|\bE\|_\rmF^2\right]}}{\sqrt{\delta}} \le \frac{13 F^3 \sqrt{U}}{\sqrt{\delta N}} \\
 & = \frac{117 \sqrt{70M} W_{\max}^4 K^4 \Delta^4 F^3 }{\sqrt{\delta N}}.
\end{align}\QED
\end{proof}

\section{Support Union Recovery}
\label{sec:union_recovery}

In this section, we first show that $\bM_2$ can also be written as $\bM_2 = \bM_2 \bX^*$, where the cardinality of the support union of $\bX^*$ is $|\mathrm{Supp}(\bX^*)| = K$. This motivates us to consider approaches for support union recovery or multiple measurement vectors~\cite{davies2012rank, ziniel2013efficient, hyder2010direction, obozinski2011support,malioutov2005sparse}. We then present theoretical guarantees for support union recovery for an $\ell_1/\ell_2$ block norm minimization problem (cf. \eqref{eq:standard_MMV} to follow).
\subsection{Another Formulation of $\bM_2$}\label{sec:M2_reform}

Recall that from~\eqref{eq:M2_form1}, we have 
\begin{equation}
 \bM_2 = \sum_{k=1}^K \alpha_k \bw_k \bw_k^T = \bW \mathrm{Diag}(\balpha) \bW^T, %= \bU \bU^T,
\end{equation}
where %$\bU := \bW \mathrm{diag}(\sqrt{\balpha})$ with 
$\balpha := [\alpha_1;\ldots;\alpha_K] \in \bbR^K$. We know that $\balpha$ contains all non-zero entries if $m_4 \ne 3m_2$. Because we assume that $\rank(\bW)=K$, there exists an index set $\scrK$ for rows of $\bW$ such that $|\scrK|=K$ and $\rank(\underline{\bW}_{\scrK}) = K$. Let $\bR \in \bbR^{K \times (F-K)}$ be the matrix such that 
\begin{equation}
 \underline{\bW}_{\scrK^c} = \bR^T \underline{\bW}_\scrK.
\end{equation}
Or equivalently, 
\begin{equation}
 \bR^T = \underline{\bW}_{\scrK^c} (\underline{\bW}_\scrK)^{-1}.
\end{equation}
Let $\mathbf{\Pi}$ be the permutation matrix corresponding to the index set $\scrK$. We have that
\begin{align}
 \bM_2 &= \bW \mathrm{Diag}(\balpha) \bW^T = \bW \mathrm{Diag}(\balpha) [\underline{\bW}_\scrK^T, \underline{\bW}_{\scrK^c}^T] \mathbf{\Pi} \\
 & = \bW \mathrm{Diag}(\balpha) [\underline{\bW}_\scrK^T, \underline{\bW}_{\scrK}^T\bR] \mathbf{\Pi} \\
 & = \bW \mathrm{Diag}(\balpha) \underline{\bW}_\scrK^T [\bI, \bR] \mathbf{\Pi} \\ 
 & = \bM_2 \bPi \begin{bmatrix}
    \bI  &   \bR			\\
     \mathbf{0}      & \mathbf{0} 		
\end{bmatrix} \bPi = \bM_2 \bX^*,
\end{align}
where $\bX^* := \bPi \begin{bmatrix}
    \bI  &   \bR			\\
     \mathbf{0}      & \mathbf{0} 		
\end{bmatrix} \bPi$. Note that the number of non-zero rows in $\bX^*$ is exactly $K$, i.e., $|\mathrm{Supp}(\bX^*)| =|\scrK|= K$. 
%\cite{chen2005sparse, lai2011null, fannjiang2011music, jin2013support, davies2012rank, ziniel2013efficient, hyder2010direction, obozinski2011support}.

\subsection{Lemmas for Support Union Recovery}

For $1 \le a\le b < \infty$ and any matrix $\bA \in \bbR^{m \times n}$, the $\ell_a/\ell_b$ block
norm of $\bA$ is defined as follows:
\begin{equation}
 %\|\bA\|_{\ell_a/\ell_b} = \left(\sum_{i=1}^m \left(\sum_{j=1}^n |A_{ij}|^b\right)^{a/b}\right)^{1/a}.
 \|\bA\|_{\ell_a/\ell_b} = \left(\sum_{i=1}^m \|\underline{\ba}_i\|_b^a\right)^{1/a},
\end{equation}
where $\underline{\ba}_i$ is the $i$-th row of $\bA$. In particular, we define 
\begin{equation}
 \|\bA\|_{\ell_\infty/\ell_2} = \max_{i \in [m]} \|\underline{\ba}_i\|_2.
\end{equation}
Assume that an observed data matrix $\bY \in \bbR^{m \times n}$ can be written as 
\begin{equation}
 \bY = \bA \bB^* + \bL,
\end{equation}
where $\bA \in \bbR^{m \times  p}$ is the dictionary matrix, $\bB^* \in \bbR^{p \times n}$ is block sparse. Let $\underline{\bb}^*_i$ be the $i$-th row of $\bB^*$, we write the support union of $\bB^*$ as $\calS:=\mathrm{Supp}(\bB^*)$.
%we define the support of $\bB^*$ as $\calS(\bB^*) := \{i \in [p]: \underline{\bb}^*_i \ne 0\}$ (the union set of support for all the columns). For simplicity of notation, we write $\calS = \calS(\bB^*)$ and denote $s = |\calS|$. 
Considering the following $\ell_1/\ell_2$ block norm minimization problem,
\begin{equation}\label{eq:standard_MMV}
 \min_{\bB \in \bbR^{p \times n}} \frac{1}{2} \|\bY - \bA \bB\|_\rmF^2 + \lambda \|\bB\|_{\ell_1/\ell_2}.
\end{equation}
%\iffalse
Note that if we denote $\underline{\bb}_i$ as the $i$-th row of $\bB$, the subdifferential of the $\ell_1/\ell_2$-block norm over row $i$ takes the form
\begin{align}
 \left[\partial \|\bB\|_{\ell_1/\ell_2}\right]_i =\left\{\begin{array}{lr}
 \frac{\underline{\bb}_i}{\|\underline{\bb}_i\|_2}, & \underline{\bb}_i \ne 0; \\
  \underline{\bz}_i \text{ such that } \|\underline{\bz}_i\|_2 \le 1, & \text{otherwise}. 
  \end{array}
\right.
\end{align}
Define the matrix $\zeta(\underline{\bB}_\calS) \in \bbR^{s \times n}$ with the $i$-th row being
\begin{equation}
 \zeta(\underline{\bb}_i) = \frac{\underline{\bb}_i}{\|\underline{\bb}_i\|_2},
\end{equation}
when $\underline{\bb}_i \ne  0$, and we set $\zeta(\underline{\bb}_i) =0$ otherwise. 
We have the following lemma by Lemma 2 in \cite{obozinski2011support}.
\begin{lemma}\label{lem:KKT_MMV}
 Suppose there exists a primal-dual pair $(\hat{\bB},\hat{\bZ})$ that satisfies the conditions
 \begin{enumerate}
  \item $\underline{\hat{\bZ}}_\calS = \zeta (\underline{\hat{\bB}}_\calS)$; 
  \item $-\lambda \underline{\hat{\bZ}}_\calS = \bA_\calS^T \bA_\calS (\underline{\hat{\bB}}_\calS - \underline{\bB}^*_\calS) - \bA_\calS^T \bL$; 
  \item $\lambda \|\underline{\hat{\bZ}}_{\calS^c}\|_{\ell_\infty / \ell_2} := \|\bA_{\calS^c}^T \bA_\calS (\underline{\hat{\bB}}_\calS - \underline{\bB}^*_\calS) - \bA_{\calS^c}^T \bL\|_{\ell_\infty / \ell_2} < \lambda$;
  \item $\underline{\hat{\bB}}_{\calS^c} = 0$.
 \end{enumerate}
Then $(\hat{\bB},\hat{\bZ})$ is a primal-dual optimal solution to the block-regularized problem~\eqref{eq:standard_MMV} with $\mathrm{Supp}(\hat{\bB}) = \calS$ by construction. If $\bA_\calS^T \bA_\calS$ is positive definite, then $\hat{\bB}$ is the unique solution. 
\end{lemma}
%\fi
Let $b^*_{\min} = \min_{i \in \calS} \|\underline{\bb}^*_i\|_2$. According to the above lemma,
%Lemma 2 in~\cite{obozinski2011support}, 
we can prove the following lemma which ensures the recovery of support union under certain conditions. 
\begin{lemma}\label{lem:opt_group_lasso}
 Assume that $\bA_\calS^T \bA_\calS$ is invertible, and there exists a fixed parameter $\gamma \in (0,1]$, such that 
 \begin{equation}
  \|\bA_{\calS^c}\bA_\calS \left(\bA_\calS^T \bA_\calS\right)^{-1}\|_\infty \le 1-\gamma.
 \end{equation}
Let $D_{\max} = \|\left(\bA_\calS^T \bA_\calS\right)^{-1}\|_\infty >0$. If 
 \begin{itemize}
  \item $D_{\max}(\lambda + \|\bA_\calS\|_1 \|\bL\|_{\ell_\infty/ \ell_2}) \le \frac{1}{2} b_{\min}^*$, 
  \item $\|\bA_{\calS^c}\|_1 \|\bL\|_{\ell_\infty/ \ell_2} \le \frac{\lambda \gamma}{2}$,
 \end{itemize}
then there is a unique optimal solution $\hat{\bB}$ for~\eqref{eq:standard_MMV} such that $\mathrm{Supp}(\hat{\bB}) = \calS$. Moreover, $\hat{\bB}$ satisfies the bound
\begin{equation}
 \|\hat{\bB} - \bB^*\|_{\ell_\infty/ \ell_2} \le D_{\max}(\lambda + \|\bA_\calS\|_1 \|\bL\|_{\ell_\infty/ \ell_2}) \le \frac{1}{2} b_{\min}^*.
\end{equation}
\end{lemma}
\begin{proof}
 We set $\underline{\hat{\bB}}_{\calS^c} = 0 $ so that the fourth condition in Lemma~\ref{lem:KKT_MMV} is satisfied. Then we set $\underline{\hat{\bB}}_{\calS}$ as the optimal solution of the following restricted version of~\eqref{eq:standard_MMV}.
 \begin{equation}
  \underline{\hat{\bB}}_{\calS} = \arg \min_{\underline{\bB}_\calS} \frac{1}{2} \|\bY - \bA_\calS \underline{\bB}_\calS\|_\rmF^2 + \lambda \|\underline{\bB}_\calS\|_{\ell_1/\ell_2}.
 \end{equation}
 Since $\bA_\calS^T \bA_\calS$ is invertible, the restricted optimization problem is strictly convex and therefore has a unique optimum $\underline{\hat{\bB}}_{\calS}$. We choose $\underline{\hat{\bZ}}_\calS$ such that the second condition in in Lemma~\ref{lem:KKT_MMV} is satisfied. Since any such $\underline{\hat{\bZ}}_\calS$ is also a dual solution to the restricted version of~\eqref{eq:standard_MMV}, it must be an element of the subdifferential $\partial \|\underline{\hat{\bB}}_{\calS}\|_{\ell_\infty/\ell_2}$. Note that by the second equality in the KKT condition  in Lemma~\ref{lem:KKT_MMV}, we have that 
 \begin{equation}
  \underline{\hat{\bB}}_\calS - \underline{\bB}^*_\calS = \left(\bA_\calS^T \bA_\calS\right)^{-1} (\bA_\calS^T \bL - \lambda \underline{\hat{\bZ}}_\calS).
 \end{equation}
Because that we have for any two matrices $\bC$ and $\bD$, $\|\bC\bD\|_{\ell_\infty/\ell_2} \le \|\bC\|_\infty \|\bD\|_{\ell_\infty/\ell_2}$, then
\begin{align}
 & \|\left(\bA_\calS^T \bA_\calS\right)^{-1} (\bA_\calS^T \bL - \lambda \underline{\hat{\bZ}}_\calS)\|_{\ell_\infty/\ell_2}  \nonumber \\
 & \le D_{\max}(\lambda + \|\bA_\calS\|_1 \|\bL\|_{\ell_\infty/ \ell_2}) \le \frac{1}{2} b_{\min}^*.
\end{align}
We have that for any $i \in \calS$, 
\begin{align}
 \|\hat{\underline{\bb}}_i\|_2 & \ge \|\underline{\bb}^*_i\|_2 - \|\left(\bA_\calS^T \bA_\calS\right)^{-1} (\bA_\calS^T \bL - \lambda \underline{\hat{\bZ}}_\calS)\|_{\ell_\infty/\ell_2} \\
 & \ge \frac{1}{2} b_{\min}^* > 0.
\end{align} 
Then we obtain that the first condition in Lemma~\ref{lem:KKT_MMV} is also satisfied. And now we also obtain the equality of the support unions and the error bound. Finally, we check whether the third condition in Lemma~\ref{lem:KKT_MMV} also holds. It is true because we have 
\begin{align}
 & \|\bA_{\calS^c}^T \bA_\calS (\underline{\hat{\bB}}_\calS - \underline{\bB}^*_\calS) - \bA_{\calS^c}^T \bL\|_{\ell_\infty / \ell_2} \nonumber \\
 & =  \|\bA_{\calS^c}^T \bA_\calS \left(\bA_\calS^T \bA_\calS\right)^{-1} (\bA_\calS^T \bL - \lambda \underline{\hat{\bZ}}_\calS)  - \bA_{\calS^c}^T \bL\|_{\ell_\infty / \ell_2} \\
 & = \| \bA^T_{\calS^c}\bPi_{\calS^\perp}(\bL) + \lambda \bA^T_{\calS^c} \bA_\calS \left(\bA_\calS^T \bA_\calS\right)^{-1} \underline{\hat{\bZ}}_\calS\|_{\ell_\infty / \ell_2} \\
 & \le \|\bA_{\calS^c}\|_1 \|\bL\|_{\ell_\infty/ \ell_2} + \lambda \|\bA^T_{\calS^c} \bA_\calS \left(\bA_\calS^T \bA_\calS\right)^{-1}\|_\infty \\
 & \le \frac{\gamma \lambda}{2} + \lambda(1-\gamma) < \lambda,
\end{align}
where $\bPi_{\calS^\perp} = \bI - \bA_\calS(\bA_\calS^T \bA_\calS)^{-1} \bA_\calS^T$ is an orthogonal projection matrix.\QED
\end{proof}

\section{The Main Theorem}
\label{sec:main_thm}
Recall that from Section~\ref{sec:M2_reform}, we obtain 
\begin{equation}
 \|\bX^*\|_{\ell_\infty/\ell_2} = \sqrt{1+\|\bR\|_{\ell_\infty/\ell_2}^2} = \sqrt{1+r_{\max}^2},
\end{equation}
where $r_{\max} = \max_{k} \|\underline{\br}_k\|_2$ with $\underline{\br}_k$ being the $k$-th row of $\bR$. In addition, let $r_{\min} = \min_{k} \|\underline{\br}_k\|_2$. We have 
\begin{equation}
 \min_{i \in \scrK} \|\underline{\bx}_{i}^*\|_2 = \sqrt{1+r_{\min}^2},
\end{equation}
where $\underline{\bx}_{i}^*$ is the $i$-th row of $\bX^*$. Let 
\begin{align}
 &\bW_1:=\underline{\bW}_\scrK \mathrm{Diag}(\balpha) \bW^T \bW \mathrm{Diag}(\balpha) \underline{\bW}_\scrK^T, \\
 & \bW_2 := \underline{\bW}_{\scrK^c}\mathrm{Diag}(\balpha)\bW^T\bW \mathrm{Diag}(\balpha) \underline{\bW}_\scrK^T.
\end{align}
We consider the $\ell_1/\ell_2$ block norm minimization problem over $\hat{\bM}_2$. 
\begin{equation}\label{eq:NMF_MMV}
 \min_{\bX \in \bbR^{F \times F}} \frac{1}{2} \|\hat{\bM}_2 - \hat{\bM}_2 \bX\|_\rmF^2 + \lambda \|\bX\|_{\ell_1/\ell_2}.
\end{equation}
We have the following main theorem which guarantees the discovering of the correct $K$.
%Let $\|\bM_2\|_1 = \max\{\|\bM_2\|_1,\|\bM_2\|_\infty\}$. 
\begin{theorem}\label{thm:main}
Let $\lambda_{\min} \left(\bW_1\right) := C_{\min} >0$, where $\lambda_{\min} \left(\bW_1\right)$ is the the minimal eigenvalue of $\bW_1$. Let $D_{\max} := \|\bW_1\|_\infty >0$. Suppose there is a $\gamma \in (0,1]$ such that 
\begin{equation}\label{eq:imp_constraint}
 \|\bW_2 \bW_1^{-1}\|_\infty \le 1-\gamma.
\end{equation}
  Let\footnote{$\zeta_1 = \gamma/\left(6\sqrt{F} \|\bM_2\|_1 D_{\max}(1+8\|\bM_2\|_1^2 D_{\max})\right)$,\\$\zeta_2 = \frac{\gamma \sqrt{1+r_{\min}^2}}{4(4+\gamma)\|\bM_2\|_1\left(1+\sqrt{F(1+r_{\max}^2)}\right)\left(D_{\max} + 6 \|\bM_2\|_1^2 D_{\max}^2\right)}$.} 
 \begin{equation}
  \zeta = \min \{\frac{\sqrt{C_{\min}}}{2}, \frac{\|\bM_2\|_1}{\sqrt{F}},\zeta_1, \zeta_2\}.
 \end{equation}
 For any $\delta \in (0,1)$, let 
 \begin{equation}
  l_{\lambda} = \frac{936 \sqrt{70M} W_{\max}^4 K^4 \Delta^4 F^3 }{\sqrt{\delta N}}\frac{\|\bM_2\|_1\left(1+\sqrt{F(1+r_{\max}^2)}\right)}{\gamma}
 \end{equation}
 and
 \begin{equation}
  u_{\lambda} = \frac{2\sqrt{1+r_{\min}^2}}{(4+\gamma)\left(D_{\max} + 6 \|\bM_2\|_1^2 D_{\max}^2\right)}.
 \end{equation}
 Then if
 \begin{equation}
  N \ge \frac{958230 M W_{\max}^2 K^8 \Delta^8 F^6 }{\delta \zeta^2},
 \end{equation}
 and 
 \begin{equation}
  l_{\lambda} \le \lambda \le u_{\lambda},
 \end{equation}
we have that with probability at least $1-\delta$, there exists a unique optimal solution $\hat{\bX}$ for~\eqref{eq:NMF_MMV} such that $\mathrm{Supp}(\hat{\bX}) = \scrK$. In addition, we have the error bound
 \begin{equation}
  \|\bX^* - \hat{\bX}\|_{\ell_\infty/\ell_2} \le \frac{\sqrt{1+r_{\min}^2}}{2}.
 \end{equation}
\end{theorem}
If the conditions of Theorem~\ref{thm:main} are satisfied, the optimal solution $\hat{\bX}$ for~\eqref{eq:NMF_MMV} satisfies that $|\mathrm{Supp}(\hat{\bX})| = K$, and thus we can count the number of non-zero rows of $\hat{\bX}$ to obtain the true $K$. The whole procedure of our algorithm is summarized in Algorithm~\ref{algo:tensor_mmv}.  
\begin{algorithm}[t]
\caption{Model selection for NMF by support union recovery}
\label{algo:tensor_mmv}
\begin{algorithmic} 
\State {\bf Input}: Data matrix $\bV \in \mathbb{R}^{F\times N}$, $\lambda>0$, $\epsilon>0$
\State {\bf Output}: The estimated value of $K$, denoted as $\hat{K}$
\State 1) Calculate $\hat{\bM}_2$ from~\eqref{eq:hatM2_calc}.
\State 2) Obtain $\hat{\bX}$ by optimizing~\eqref{eq:NMF_MMV}.
\State 3) $\hat{K} := |\{f \in [F]: \|\hat{\underline{\bx}}_f\|_2 > \epsilon\}|$.
\end{algorithmic}
\end{algorithm} 

\subsection{The Proof of Theorem~\ref{thm:main}}

Before presenting the proof of our main theorem, we provide following useful lemmas. The next lemma can be found in~\cite{golub2012matrix}, and it is used in Lemma~\ref{lem:tedious_compt} to follow.
\begin{lemma}\label{lem:dist_inverse}
 If $\bA$ is invertible and $r=\|\bA^{-1}\bL\|_\infty <1$, then $\bA+\bL$ is invertible and $\|(\bA+\bL)^{-1}-\bA^{-1}\|_\infty \le \frac{\|\bL\|_\infty \|\bA^{-1}\|_\infty^2}{1-r}$
\end{lemma}
%Then we have the following lemma.
The following lemma says that if the perturbation is sufficiently small, the perturbed matrix inherits certain nice properties of the original matrix. 
\begin{lemma}\label{lem:tedious_compt}
 Assume that we have $\bY = \bar{\bY} + \bL \in \bbR^{m\times n}$. Let $\calS$ be the index set of columns in $\bar{\bY}$ such that $\bar{\bY}_\calS^T \bar{\bY}_\calS$ is invertible. Furthermore, assume that the minimal eigenvalue of $\bar{\bY}_\calS^T \bar{\bY}_\calS$ is $\lambda_{\min} \left(\bar{\bY}_\calS^T \bar{\bY}_\calS\right) = C_{\min} >0$, and there is a $\gamma \in (0,1]$ such that $\|\bar{\bY}_{\calS^c}^T \bar{\bY}_\calS \left(\bar{\bY}_\calS^T \bar{\bY}_\calS\right)^{-1}\|_\infty \le 1-\gamma$. Let $D_{\max} = \|\left(\bar{\bY}_\calS^T \bar{\bY}_\calS\right)^{-1}\|_\infty >0$, $U_{\bL}=\max\{\|\bL\|_1,\|\bL\|_\infty\}$, and $U_{\bar{\bY}}=\max\{\|\bar{\bY}\|_1,\|\bar{\bY}\|_\infty\}$.
 Then if the noise matrix $\bL$ satisfies that 
 \begin{itemize}
  \item $\|\bL_\calS\|_2 \le \frac{\sqrt{C_{\min}}}{2}$;
  \item $\eta:=\|\left(\bar{\bY}_\calS^T \bar{\bY}_\calS\right)^{-1}\left(\bL_\calS^T\bY_\calS + \bY_\calS^T \bL_\calS + \bL_\calS^T \bL_\calS\right)\|_\infty <1$;
  \item $D_{\max}U_{\bL} (2U_{\bar{\bY}}+ U_\bL)\left[1+\frac{(U_{\bar{\bY}}+ U_\bL)^2 D_{\max}}{1-\eta}\right] \le \frac{\gamma}{2}$;
 \end{itemize}
 we have that the following conditions are satisfied:
 \begin{itemize}
  \item $\bY_\calS^T \bY_\calS$ is invertible; 
  \item $ \|\bY_{\calS^c}^T \bY_\calS \left(\bY_\calS^T \bY_\calS\right)^{-1}\|_\infty \le 1-\frac{\gamma}{2}$;
  \item $\|\left(\bY_\calS^T \bY_\calS\right)^{-1} \|_\infty \le D_{\max}^*$,
 \end{itemize}
 where $D_{\max}^* = D_{\max} + \frac{U_{\bL}(2U_{\bar{\bY}}+ U_{\bL}) D_{\max}^2}{1-\eta}$.
\end{lemma}
\begin{proof}
 We have that 
 \begin{equation}
  |\sigma_{|\calS|}(\bY_\calS) - \sigma_{|\calS|}(\bar{\bY}_\calS)| \le \|\bL_\calS\|_2 \le \frac{\sqrt{C_{\min}}}{2},
 \end{equation}
where $\sigma_{|\calS|}(\cdot)$ denotes the $|\calS|$-th largest singular value of a matrix. Because $\sigma_{|\calS|}(\bar{\bY}_\calS) = \sqrt{C_{\min}}$, we have that $\sigma_{|\calS|}(\bY_\calS) \ge \frac{\sqrt{C_{\min}}}{2}$. Or equivalent, $\bY_\calS^T \bY_\calS$ is invertible with $\lambda_{\min}\left(\bY_\calS^T \bY_\calS\right) \ge \frac{C_{\min}}{4}$. In addition, by Lemma~\ref{lem:dist_inverse}, if $\eta <1$, we have that
\begin{align}
 & \|\left(\bY^T_\calS \bY_\calS\right)^{-1} - \left(\bar{\bY}_\calS^T \bar{\bY}_\calS\right)^{-1}\|_\infty \nonumber \\
 & \le \frac{\|\bL_\calS^T\bY_\calS + \bY_\calS^T \bL_\calS + \bL_\calS^T \bL_\calS\|_\infty D_{\max}^2}{1-\eta} \\
 & \le \frac{U_{\bL}(2U_{\bar{\bY}}+ U_{\bL}) D_{\max}^2}{1-\eta}. \label{eq:bound_dist_inverse}
\end{align}
And we have $\|\left(\bY^T_\calS \bY_\calS\right)^{-1}\|_\infty \le \frac{U_{\bL}(2U_{\bar{\bY}}+ U_{\bL}) D_{\max}^2}{1-\eta} + D_{\max} =D_{\max}^*$.
Then we have 
\begin{align}
 & \|\bY_{\calS^c}^T \bY_\calS \left(\bY_\calS^T \bY_\calS\right)^{-1} - \bar{\bY}_{\calS^c}^T \bar{\bY}_\calS \left(\bar{\bY}_\calS^T \bar{\bY}_\calS\right)^{-1}\|_\infty \nonumber \\
 & \le \|\left(\bY_{\calS^c}^T - \bar{\bY}_{\calS^c}^T\right)\bar{\bY}_\calS \left(\bar{\bY}_\calS^T \bar{\bY}_\calS\right)^{-1}\|_\infty  \nonumber\\
 & + \|\bY_{\calS^c}^T \left(\bY_\calS \left(\bY_\calS^T \bY_\calS\right)^{-1}  - \bar{\bY}_\calS \left(\bar{\bY}_\calS^T \bar{\bY}_\calS\right)^{-1}\right)\|_\infty \\
 & \le U_{\bL} U_{\bar{\bY}} D_{\max} +  \nonumber\\
 & (U_{\bL}+U_{\bar{\bY}}) \|\bY_\calS \left(\bY_\calS^T \bY_\calS\right)^{-1}  - \bar{\bY}_\calS \left(\bar{\bY}_\calS^T \bar{\bY}_\calS\right)^{-1}\|_\infty \\
 & \le U_{\bL} U_{\bar{\bY}} D_{\max} + (U_{\bL}+U_{\bar{\bY}}) \times  \nonumber\\ 
 & [\|(\bY_\calS - \bar{\bY}_\calS)\left(\bar{\bY}_\calS^T \bar{\bY}_\calS\right)^{-1}\|_\infty + \nonumber\\
 & \|\bY_\calS \left(\bY^T_\calS \bY_\calS\right)^{-1} - \left(\bar{\bY}_\calS^T \bar{\bY}_\calS\right)^{-1}\|_\infty] \\
 & \le U_{\bL} U_{\bar{\bY}} D_{\max} + (U_{\bL}+U_{\bar{\bY}}) \times  \nonumber \\
 & \left(U_{\bL}D_{\max} + (U_{\bL}+U_{\bar{\bY}}) \|\left(\bY^T_\calS \bY_\calS\right)^{-1} - \left(\bar{\bY}_\calS^T \bar{\bY}_\calS\right)^{-1}\|_\infty\right)\\
 & \le D_{\max}U_{\bL} (2U_{\bar{\bY}}+ U_\bL)\left[1+\frac{(U_{\bar{\bY}}+ U_\bL)^2 D_{\max}}{1-\eta}\right].
\end{align}
If $D_{\max}U_{\bL} (2U_{\bar{\bY}}+ U_\bL)\left[1+\frac{(U_{\bar{\bY}}+ U_\bL)^2 D_{\max}}{1-\eta}\right] \le \frac{\gamma}{2}$, we obtain that 
\begin{equation}
 \|\bY_{\calS^c}^T \bY_\calS \left(\bY_\calS^T \bY_\calS\right)^{-1}\|_\infty \le 1-\frac{\gamma}{2}.
\end{equation}\QED
\end{proof}

Now we present the proof of our main theorem. 
\begin{proof}[Proof of Theorem~\ref{thm:main}]
 Note that $\hat{\bM}_2 = \bM_2 + \bE$, and $\|\bE\|_2 \le \|\bE\|_\rmF$, $\|\bE\|_{\ell_\infty/ \ell_2}  \le \|\bE\|_\rmF$, $\|\bE\|_1  = \|\bE\|_\infty \le \sqrt{F}\|\bE\|_\rmF$. %Let $\bE_\scrK, (\bM_2)_\scrK, (\hat{\bM}_2)_\scrK$ be the matrix of the columns corresponding to index set $\scrK$ in $\bE, \bM_2, \hat{\bM}_2$ respectively. 
 Then if 
 \begin{equation}
  \|\bE\|_\rmF \le \min\{ \frac{\sqrt{C_{\min}}}{2}, \frac{\|\bM_2\|_1}{\sqrt{F}},\zeta_1 \},
 \end{equation}
 we have that 
 \begin{itemize}
  \item $\|\bE_\scrK\|_2 \le \frac{\sqrt{C_{\min}}}{2}$;
  \item Let $\bT = \bE_\scrK^T(\hat{\bM}_2)_\scrK + (\hat{\bM}_2)_\scrK^T \bE_\scrK + \bE_\scrK^T \bE_\scrK$, we have that $\eta :=\|\left((\bM_2)_\scrK^T (\bM_2)_\scrK\right)^{-1}\bT\|_\infty \le \frac{1}{2}<1$;
  \item $D_{\max}\|\bE\|_1 (2\|\bM_2\|_1+ \|\bE\|_1)\left[1+\frac{(\|\bM_2\|_1+ \|\bE\|_1)^2 D_{\max}}{1-\eta}\right] \le \frac{\gamma}{2}$.
 \end{itemize}
 By Lemma~\ref{lem:tedious_compt}, the following conditions are satisfied:
 \begin{itemize}
  \item $(\hat{\bM}_2)_\scrK^T (\hat{\bM}_2)_\scrK$ is invertible; 
  \item $ \|(\hat{\bM}_2)_{\scrK^c}^T (\hat{\bM}_2)_\scrK \left((\hat{\bM}_2)_\scrK^T (\hat{\bM}_2)_\scrK\right)^{-1}\|_\infty \le 1-\frac{\gamma}{2}$;
  \item $\|\left((\hat{\bM}_2)_\scrK^T (\hat{\bM}_2)_\scrK\right)^{-1} \|_\infty \le D_{\max}^*$,
 \end{itemize}
 where $D_{\max}^* = D_{\max} + \frac{\|\bE\|_1(2\|\bM_2\|_1+ \|\bE\|_1) D_{\max}^2}{1-\eta}$.
 
 In addition, note that $\hat{\bM}_2 = \bM_2 + \bE = \bM_2 \bX^* + \bE = \hat{\bM}_2 \bX^* + \bE^*$, where $\bE^*:= \bE + \bE \bX^*$. We have 
 \begin{align}
  &\|\bE^*\|_{\ell_\infty/\ell_2} \le \|\bE\|_{\ell_\infty/\ell_2} + \|\bE\|_1 \|\bX^*\|_{\ell_\infty/\ell_2} \\
  &\le \|\bE\|_{\rmF} + \sqrt{F}\|\bE\|_{\rmF} \sqrt{1+r_{\max}^2}.
 \end{align}
When $\|\bE\|_\rmF \le \frac{\|\bM_2\|_1}{\sqrt{F}}$, we have 
\begin{equation}
 D_{\max}^* \le D_{\max} + 6 \|\bM_2\|_1^2 D_{\max}^2.
\end{equation}
Then if 
\begin{equation}
 \|\bE\|_\rmF \le \zeta_2,
\end{equation}
we can choose $\lambda$ such that
\begin{equation}
 \frac{8\|\bM_2\|_1\|\bE\|_\rmF\left(1+\sqrt{F}\sqrt{1+r_{\max}^2}\right)}{\gamma} \le \lambda \le u_{\lambda}.
\end{equation}
For such $\lambda$, we obtain that
\begin{align}
        & D_{\max}^* (\lambda + (\|\bE\|_1 + \|\bM_2\|_1)(\|\bE\|_{\ell_\infty/\ell_2} + \|\bE\|_1\|\bX^*\|_{\ell_\infty/\ell_2})) \nonumber \\
        & \le \frac{1}{2} \min_{i \in \scrK} \|\underline{\bx}_{i}^*\|_2 = \frac{1}{2}\sqrt{1+r_{\min}^2}
       \end{align}
and 
\begin{equation}
 (\|\bE\|_1 + \|\bM_2\|_1) (\|\bE\|_{\ell_\infty/\ell_2} + \|\bE\|_1\|\bX^*\|_{\ell_\infty/\ell_2}) \le \frac{\lambda\gamma}{4}.
\end{equation}
Then by Lemmas~\ref{lem:bd_bE} and~\ref{lem:opt_group_lasso}, for $N$ and $\lambda$ in the ranges given in the statement of the theorem, there exists a unique optimal solution $\hat{\bX}$ for~\eqref{eq:NMF_MMV} such that $\mathrm{Supp}(\hat{\bX}) = \scrK$ with the desired error bound. \QED
\end{proof}

\section{Numerical Results}
\label{sec:num_res}
To demonstrate the efficacy of Algorithm~\ref{algo:tensor_mmv} for estimating $K$, we perform numerical simulations on synthetic datasets. We need to obtain $\hat{\bM}_2$ in the first step of Algorithm~\ref{algo:tensor_mmv}. The time complexity of calculating $\hat{\bM}_4$ is $O(F^4 N)$. However, note that we do not need to calculate $\hat{\bM}_4$ explicitly before calculating $\hat{\bM}_2$. Let $p_n = \sum_{f=1}^F v_{f,n}$ and $q_n = p_n^2$ for $n \in [N]$. Let $\hat{\calA} = \frac{\sum_{n=1}^N \bv_n \otimes \bv_n \otimes \bv_n \otimes \bv_n }{N}$. We have that 
\begin{align}
 &\sum_{l =1}^F \sum_{m=1}^F [\hat{\calA}]_{ijlm}  \nonumber\\
 &= \frac{\sum_{n=1}^N v_{i,n}v_{j,n} (\sum_{l=1}^F v_{l,n}) (\sum_{m=1}^F v_{m,n})}{N} \nonumber\\
 &= \frac{\sum_{n=1}^N p_n^2 v_{i,n}v_{j,n}}{N}.
\end{align}
Let $\bS = (s_{ij}) \in \bbR^{F\times F}$ be the matrix such that 
\begin{equation}
 s_{ij} = \sum_{l =1}^F \sum_{m=1}^F [\hat{\calA}]_{ijlm}.
\end{equation}
We have that
\begin{equation}
 \bS = \frac{\bV \mathrm{Diag}(\bq)\bV^T}{N}. 
\end{equation}
Considering similar reformulations for the summation over components of $\hat{\calT}$, we obtain that 
\begin{equation}\label{eq:hatM2_calc}
 \hat{\bM}_2 = \frac{\bV \mathrm{Diag}(\bq)\bV^T}{N} - \left(\frac{\sum_{n=1}^N q_n}{N^2} \bV\bV^T + 2\frac{\bV \bp \bp^T \bV^T}{N^2}\right).
\end{equation}
The time complexity for calculating $\hat{\bM}_2$ is reduced to $O(F^2 N)$.
In the second step of Algorithm~\ref{algo:tensor_mmv}, we use CVX~\cite{cvx} to obtain a solution $\hat{\bX}$ for~\eqref{eq:NMF_MMV}. $\epsilon$ is set to be $10^{-6}$ in the third step of Algorithm~\ref{algo:tensor_mmv}.
 %to obtain $\hat{\bX}$ by optimizing~\eqref{eq:NMF_MMV}, we implement the $\ell_1$-SVD method~\cite{malioutov2005sparse} using CVX~\cite{cvx}. 

\subsection{Synthetic Datasets}
 We fix $K=10$ and vary $F$ between 20 and 50. We vary the number of samples $N$ from 100 to 10000. We set the dictionary matrix $\bW \in \bbR^{F\times K}$ as $[\bI_K; \tau \bW_c]$, where $\bI_K$ is the identity matrix in $\bbR^{K \times K}$ and $\bW_c \in \bbR^{(F-K) \times K}$ is a random non-negative matrix generated from the command \texttt{rand(F-K,K)} in Matlab. $\tau >0$ is properly chosen such that $\|\bW_2 \bW_1^{-1}\|_\infty < 1$ (cf. \eqref{eq:imp_constraint}). Each entry $h$ of $\bH$ is generated from %a Bernoulli distribution with parameter $p = 0.2$ or from 
an exponential distribution\footnote{$\mathrm{Exp} (u)$ is the function $x\mapsto u\exp(- u x)1\{x\ge 0\}$.} $\mathrm{Exp} (u)$  with parameter $u = 1$, and then is centralized by $h \leftarrow h - \frac{1}{u}$. %For Bernoulli distributions, the regularization parameter $\lambda$ for~\eqref{eq:NMF_MMV} is set to be $1$. While 
The regularization parameter $\lambda$ %for~\eqref{eq:NMF_MMV} 
is set to be $10$. The data matrix $\bV = \bW \bH + \bZ$ and each entry of the noise matrix $\bZ$ is sampled from a Gaussian distribution $\calN(0,\sigma^2)$ with $\sigma = 0.01$. For each setting of the parameters, we generate 20 data matrices $\bV$ independently. %From Fig.~\ref{fig:est_K_bino}, we observe that Algorithm~\ref{algo:tensor_mmv} almost always detects the true $K$ even if the number of samples $N$ is relatively small. 
From Fig.~\ref{fig:est_K_exp}, we observe that when $F = 20$, the algorithm cannot detect the true $K$ until $N$ is sufficiently large (e.g., $N \ge 6 \times10^3$). When $F=50$, we need a smaller $\tau$ such that $\|\bW_2 \bW_1^{-1}\|_\infty < 1$, and the algorithm works well even when the sample size is relatively small. 

\iffalse
\begin{figure}[t]
\subfloat{\includegraphics[width=.525\columnwidth]{figures/bino_estK_mat_F20K10_sigma1_p2_lambdaReg1_20181019}}\hfill
%\subfloat{\includegraphics[width=.325\columnwidth]{figures/bino_estK_mat_F35K10_sigma1_p2_lambdaReg1_20181019}}\hfill
\subfloat{\includegraphics[width=.475\columnwidth]{figures/bino_estK_mat_F50K10_sigma1_p2_lambdaReg1_20181019}}
\caption{Estimated number of components $K$ with different $F$ for Bernoulli distribution. The error bars denote one standard deviation away from the mean.  }\label{fig:est_K_bino}
\end{figure}
\fi
\begin{figure}[t]
\subfloat{\includegraphics[width=.5\columnwidth]{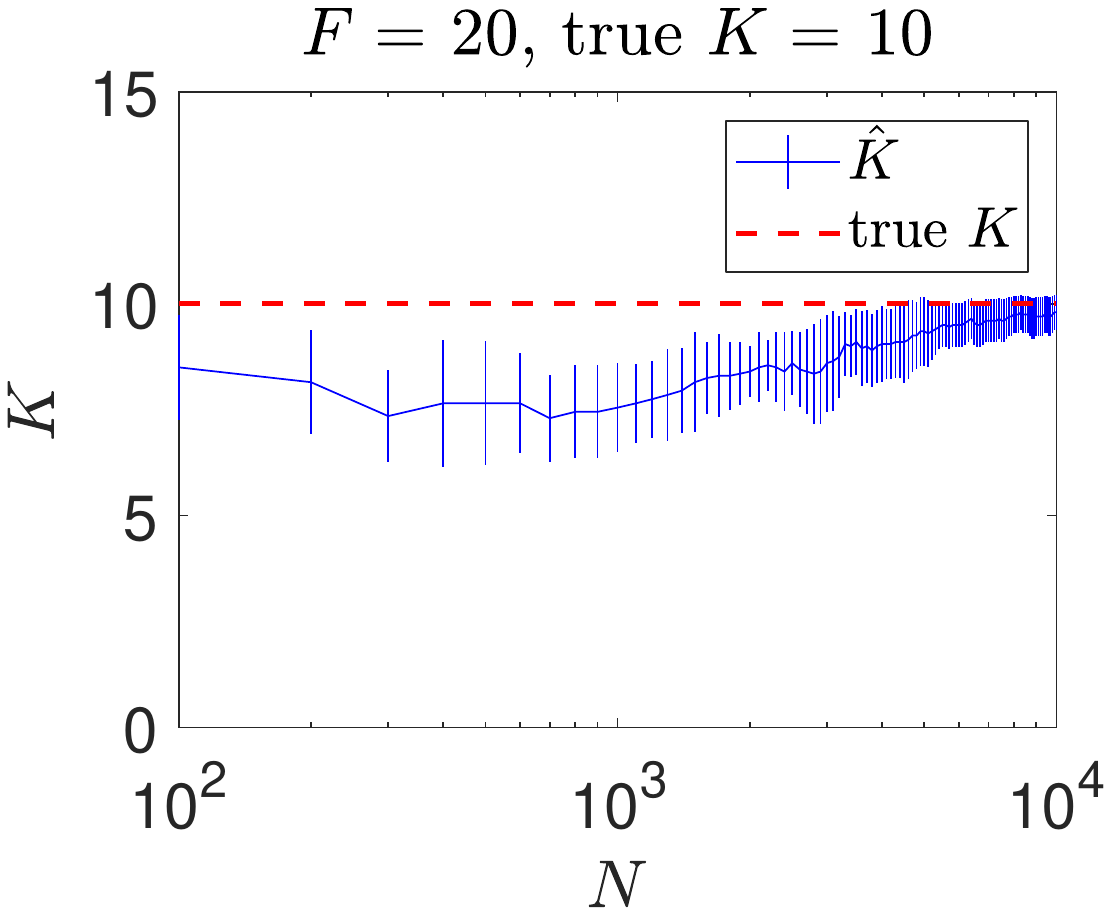}}\hfill
%\subfloat{\includegraphics[width=.325\columnwidth]{figures/bino_estK_mat_F35K10_sigma1_p2_lambdaReg1_20181019}}\hfill
\subfloat{\includegraphics[width=.475\columnwidth]{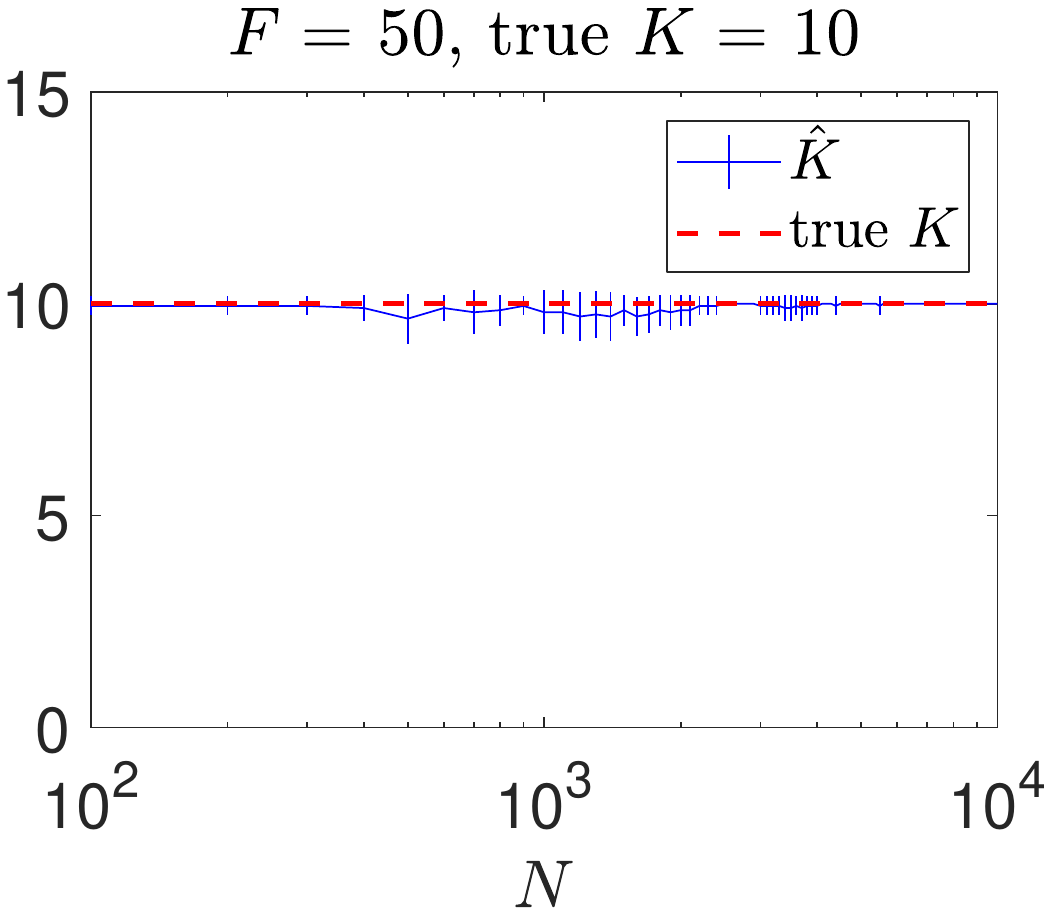}}
\caption{Estimated number of components $K$ with different $F$ for exponential distribution. The error bars denote one standard deviation away from the mean.}\label{fig:est_K_exp}
\end{figure}

\subsection{The Swimmer Dataset}
We perform experiments on the well-known %iris~\cite{Dua:2017} and 
swimmer~\cite{gillis2014robust} dataset, which is widely used for benchmarking NMF algorithms. %The iris dataset contains $150$ samples for three different species of iris. Samples in one species are well separately with those in another two species, and thus the number of clusters can be viewed as either two or three. Clustering can be considered as a special type of matrix factorization problem~\cite{liu2017informativeness}, and the number of clusters can be regarded as the latent dimensionality for a clustering problem. The number of features is $F=4$. We plot the regularization path in Fig.~\ref{}. The regularization path shows ... 
The swimmer dataset we use contains $256$ binary images ($20$-by-$11$ pixels) which depict figures with four limbs, each can be in four different positions. The %rank of the corresponding data matrix is $13$ and the nonnegative rank (i.e., the latent dimensionality) is $16$.
latent dimensionality of the corresponding data matrix is $16$.
From the regularization path for this dataset presented in Fig.~\ref{fig:swimmer}, we observe that the estimated latent dimensionality $\hat{K}$ is always $14$ when $10^{-5}\le\lambda \le 10^9$. In addition, the relative error $\frac{\|\hat{\bM}_2 - \hat{\bM}_2\hat{\bX}\|_\rmF}{\|\hat{\bM}_2\|_\rmF}$ is close to $0$ when $\lambda\le 10^4$ and becomes intolerably large (larger than $0.75$) when $\lambda \ge 10^8$. Therefore, a reasonable estimate for the latent dimensionality is $14$, which is close to the true latent dimensionality.

\begin{figure}[t]
\subfloat{\includegraphics[width=.475\columnwidth]{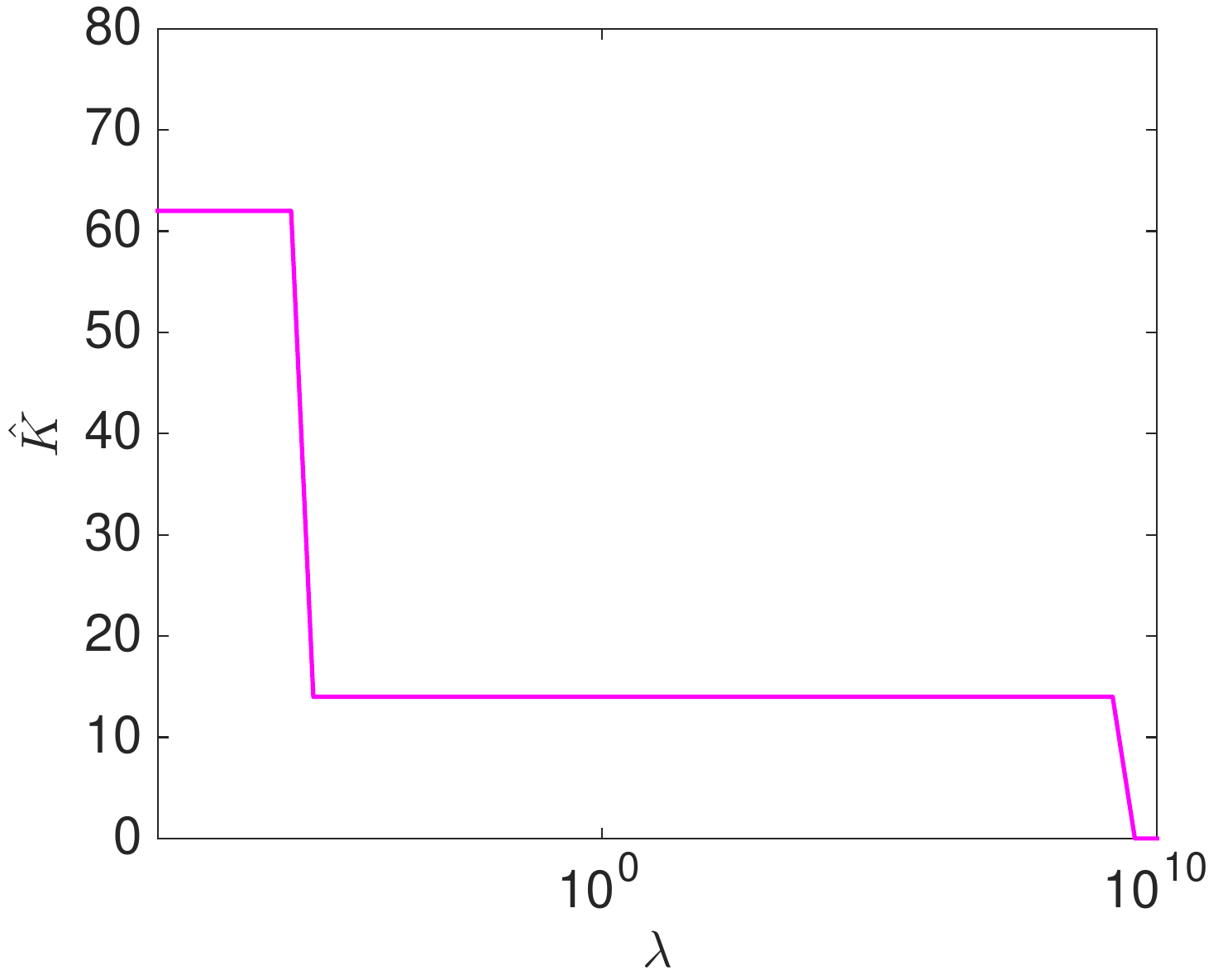}}\hfill
%\subfloat{\includegraphics[width=.325\columnwidth]{figures/bino_estK_mat_F35K10_sigma1_p2_lambdaReg1_20181019}}\hfill
\subfloat{\includegraphics[width=.475\columnwidth]{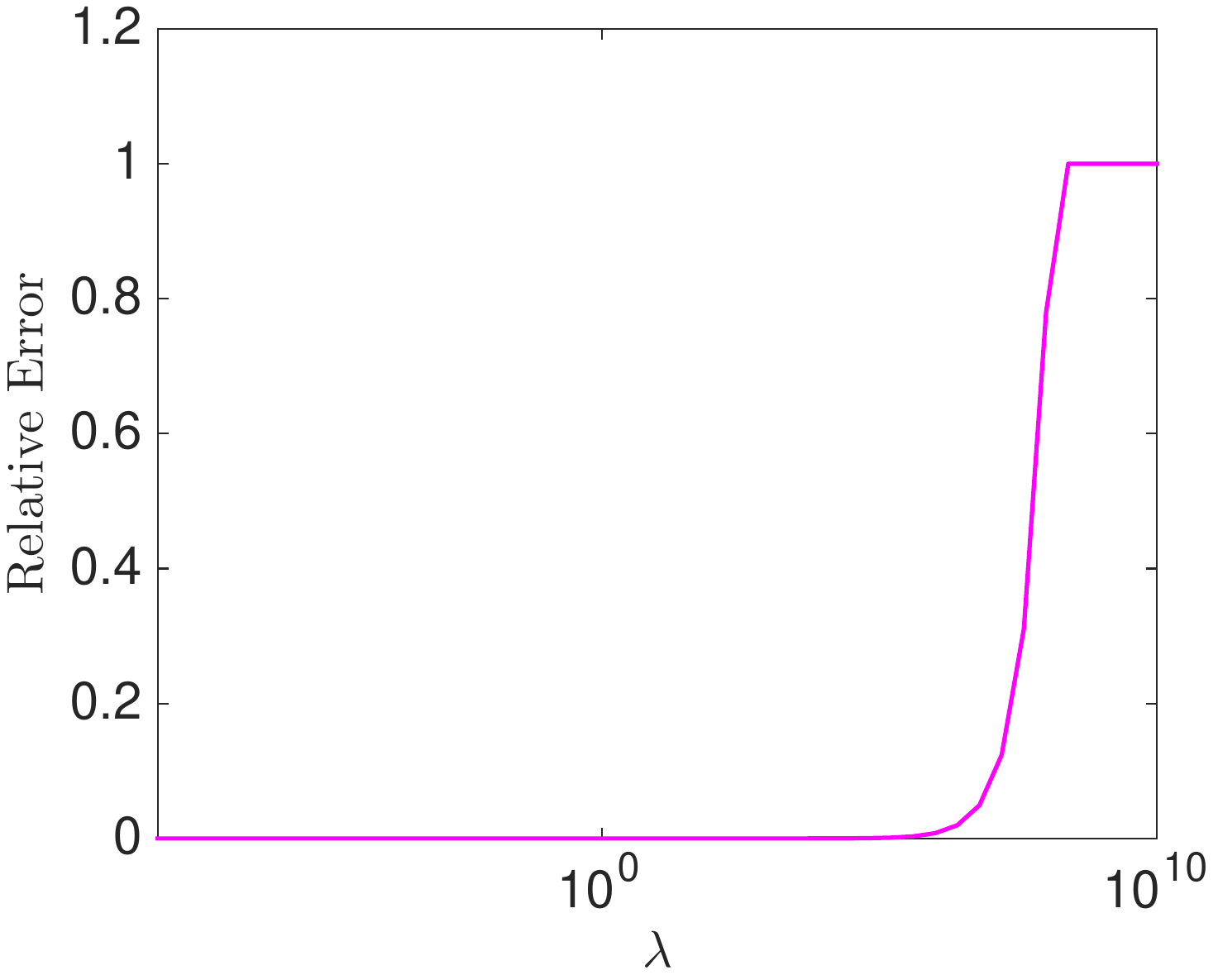}}
\caption{Estimated number of components and relative error with different $\lambda$ for the swimmer dataset.}\label{fig:swimmer}
\end{figure}

\section{Future Work}
\label{sec:futureWork}

We would like to extend our theoretical analysis to other latent variable models, such as Gaussian mixture models~\cite{titterington1985,liu2017informativeness} and latent Dirichlet allocation~\cite{blei2003latent, cheng2015model}. The parameter $F$ plays an important role in the time complexity of Algorithm~\ref{algo:tensor_mmv} and $F$ is very large for certain real data. We may consider combining dimensionality reduction techniques with Algorithm~\ref{algo:tensor_mmv} to reduce the running time. Finally, we hope to provide sufficient conditions for the existence of the index set $\scrK$ such that $\|\bW_2 \bW_1^{-1}\|_\infty < 1$ (cf. \eqref{eq:imp_constraint}). 

\iffalse
\begin{itemize}
 \item To perform more experiments and make the presentation of the numerical section nicer;
 \item To check whether we can modify the lemma for fourth order tensor for unnormalized rvs, because we want $\bh$ to be non-negative.
\end{itemize}
\fi

%\vfill
%\pagebreak

%\vfill\pagebreak

\bibliographystyle{IEEEbib}
\bibliography{modelSel_NMF}
\end{document}